\documentclass[english,final]{siamltex}
\usepackage[T1]{fontenc}
\usepackage[latin9]{inputenc}
\usepackage{amssymb}
\usepackage{amsmath}
\usepackage{graphicx}
\usepackage{esint}
\usepackage{babel}
\usepackage{bm}

\def\bu{{\mathbf{u}}}
\def\bv{{\mathbf{v}}}
\def\bx{{\mathbf{x}}}
\def\by{{\mathbf{y}}}
\def\bz{{\mathbf{z}}}
\def\bt{{\mathbf{t}}}

\def\bf{{\mathbf{f}}}

\def\bomega{{\bm{\omega}}}
\def\bnu{{\bm{\nu}}}

\def\bxo{{\mathbf{x}_0}}
\def\Au{{\mathcal{A}_{\bu}}}
\def\Neta{{\mathcal{N}_{\eta}}}
\title{Gradient density estimation in arbitrary finite dimensions using the method of stationary phase}

\author{
Karthik S. Gurumoorthy\footnotemark[4]
\thanks{Present address: International Center for Theoretical Sciences, Tata Institute of Fundamental Research, TIFR Centre Building, Indian Institute of Science Campus, Bangalore, Karnataka, 560012, India. Ph:+91-80-23610109 (extn: 21). This work is partly supported by EADS Prize Postdoctoral Fellowship. Email: {\tt karthik.gurumoorthy@icts.res.in}}
\and
Anand Rangarajan\footnotemark[4]
\thanks{Email: {\tt anand@cise.ufl.edu}}
\and
John Corring\footnotemark[4] 
\thanks{Email: {\tt corring@cise.ufl.edu}}
}

\begin{document}
\maketitle
\renewcommand{\thefootnote}{\fnsymbol{footnote}}
\footnotetext[4]{Department of Computer and Information Science and Engineering,
University of Florida, Gainesville, Florida, USA}
\begin{abstract}
We prove that the density function of the gradient of a sufficiently smooth
function $S : \Omega \subset \mathbb{R}^d \rightarrow \mathbb{R}$, obtained
via a random variable transformation of a uniformly distributed random
variable, is increasingly closely approximated by the normalized power spectrum of
$\phi=\exp\left(\frac{iS}{\tau}\right)$ as the free parameter
$\tau\rightarrow0$. The result is shown using the stationary phase approximation
and standard integration techniques and requires proper ordering of limits. We
highlight a relationship with the well-known characteristic function approach
to density estimation, and detail why our result is distinct from this
approach.
\end{abstract}
\begin{keywords}
\emph{Keywords:} stationary phase approximation; density estimation; Fourier
transform; power spectrum; wave functions; characteristic functions
\end{keywords}
\begin{AMS}
 Primary 41A60, 62G07; Secondary 42B10
\end{AMS}

\pagestyle{myheadings}
\thispagestyle{plain}
\markboth{ KARTHIK S. GURUMOORTHY, ANAND
  RANGARAJAN AND JOHN CORRING}{GRADIENT DENSITY ESTIMATION IN ARBITRARY FINITE DIMENSIONS}
  
\section{Introduction}
\indent Density estimation methods provide a faithful estimate of a
non-observable probability density function based on a given collection of
observed data \cite{Parzen62,Rosenblatt56,Silverman86,Bishop06}. The observed data are
treated as random samples obtained from a large population which is assumed to
be distributed according to the underlying density function. The aim of our
current work is to show that the joint density function of the gradient of a
sufficiently smooth function $S$ (density function of $\nabla S$) can be
obtained from the normalized power spectrum of
$\phi=\exp\left(\frac{iS}{\tau}\right)$ as the free parameter $\tau$ tends to
zero. The proof of this relationship relies on the higher order stationary
phase approximation \cite{Jones58,McClure91,McClure97,Wong89,Wong81}. The joint density function of the
gradient vector field is usually obtained via a random variable transformation
of a uniformly distributed random variable $X$ over the compact domain $\Omega
\subset \mathbb{R}^d$ using $\nabla S$ as the transformation function. In
other words, if we define a random variable $Y= \nabla S(X)$ where the random
variable $X$ has a \emph{uniform distribution} on the $\Omega$ ($ X \sim
UNI(\Omega)$), the density function of $Y$ represents the joint density
function of the gradient of $S$.

In computer vision parlance---a popular application area for density
estimation---these gradient density functions are popularly known as the
histogram of oriented gradients (HOG) and are primarily employed for human and
object detection \cite{Dalal05,Zhu06}. The approaches developed in
\cite{Suard06,Bertozzi07} demonstrate an application of HOG---in combination
with support vector machines \cite{Vapnik98}---for detecting pedestrians from
infrared images. In a recent article \cite{Hu13}, an adaption of the HOG
descriptor called the Gradient Field HOG (GF-HOG) is used for sketch based
image retrieval. In these systems, the image intensity is treated as a
function $S(X)$ over a $2D$ domain, and the distribution of intensity
gradients or edge directions is used as the feature descriptor to characterize
the object appearance or shape within an image.

 In our earlier effort \cite{Gurumoorthy12}, we primarily focused on
 exploiting the stationary phase approximation to obtain gradient densities of
 Euclidean distance functions ($R$) in two dimensions. As the gradient norm of
 $R$ is identically equal to $1$ everywhere, the density of the gradient is
 one-dimensional and defined over the space of orientations. The key point to
 be noted here is that the dimensionality of the gradient density (one) is one
 less than the dimensionality of the space (two) and the constancy of the
 gradient magnitude of $R$ causes its Hessian to vanish almost everywhere. In
 Lemma~\ref{lemma:densityLemma} below, we see that the Hessian is deeply connected to
 the density function of the gradient. The degeneracy of the Hessian
 precluded us from directly employing the stationary phase method and hence
 techniques like symmetry-breaking had to be used to circumvent the vanishing
 Hessian problem. The reader may refer to \cite{Gurumoorthy12} for a more
 detailed explanation.  In contrast to our previous work, we regard our
 current effort as a generalization of the gradient density estimation result,
 now established for \emph{arbitrary} smooth functions in \emph{arbitrary}
 finite dimensions.

 \subsection{Main Contribution} 
 We introduce a new approach for computing the density of $Y$, where we
 express the given function $S$ as the phase of a wave function $\phi$,
 specifically $\phi(\bx)=\exp\left(\frac{iS(\bx)}{\tau}\right)$ for small
 values of $\tau$, and then consider the normalized power spectrum---magnitude
 squared of the Fourier transform---of $\phi$ \cite{Bracewell99}.  We show
 that the computation of the joint density function of $Y=\nabla S$ may be
 approximated by the power spectrum of $\phi$ with the approximation becoming
 increasingly tight point-wise as $\tau \rightarrow 0$. Using the stationary
 phase approximation ---a well known technique in asymptotic analysis
 \cite{Wong89}---we show that in the limiting case as $\tau\rightarrow0$, the
 power spectrum of $\phi$ converges to the density of $Y$ and hence can serve
 as its density estimator at small, non-zero values of $\tau$.  In other
 words, if $P(\bu)$ denotes the density of $Y$ and if $P_{\tau}(\bu)$
 corresponds to the power spectrum of $\phi$ at a given value of $\tau$,
 Theorem~\ref{thm:GradDensity} constitutes the following relation,
 \begin{equation}
\lim\limits_{\tau \rightarrow 0}  \int\limits_{\overline{\Neta(\bu_0)}}P_\tau (\bu) d\bu =  \int\limits_{\overline{\Neta(\bu_0)}} P(\bu) d\bu
\end{equation}
where $\overline{\Neta(\bu_0)}$ is a small neighborhood around $\bu_0$. We
call our approach the \emph{wave function method} for computing the
probability density function and henceforth will refer to it as such. We would
like to emphasize that our work is \emph{fundamentally different} from
estimating the gradient of a density function \cite{Fukunaga75} and should
not be semantically confused with it.

 \subsection{Significance}
 As mentioned before, the main objective of our current work is to generalize
 our effort in \cite{Gurumoorthy12} and demonstrate the fact that the wave
 function method for obtaining densities can be extended to arbitrary
 functions in arbitrary finite dimensions. However, one might broach a
 legitimate question, namely \emph{``What is the primary advantage of this
   approach over other simpler, effective and traditional techniques like
   histograms which can compute the HOG say by mildly smoothing the image,
   computing the gradient via (for example) finite differences and then binning the
   resulting gradients?''}. The benefits are three fold:
 \begin{itemize}
 \item One of the foremost advantages of our wave function approach is that it
   recovers the joint gradient density function of $S$ \emph{without}
   explicitly computing its gradient.  Since the stationary points capture
   gradient information and map them into the corresponding frequency bins, we
   can directly work with $S$ without the need to compute its derivatives.
\item The significance of our work is highlighted when we deal with the
  more practical finite sample-set setting wherein the gradient density is
  estimated 
  from a \emph{finite, discrete} set of samples of $S$ rather than assuming the
  availability of the complete description of $S$ on $\Omega$. Given the $N$
  samples of $S$ on $\Omega$, it is customary to know the \emph{rate of
    convergence} of a proposed density estimation method as $N \rightarrow \infty$. In
  \cite{Gurumoorthy14} we prove that in one dimension, our wave function
  method converges point-wise at the rate of $O(1/N)$ to the true density when $\tau \propto
  1/N$. For histograms and the kernel density estimators \cite{Parzen62,Rosenblatt56}, the convergence rates are established for the integrated mean squared error (IMSE) defined as the expected value (with respect to samples of size $N$) of the square of the $\ell_2$ error between the true and the computed probability densities and are shown to be $O\left(N^{-\frac{2}{3}}\right)$ \cite{Scott79,Cencov62} and $O\left(N^{-\frac{4}{5}}\right)$ \cite{Wahba75} respectively. Having laid the foundation in this work, we plan to invest our future efforts in pursuit of similar convergence estimates in arbitrary finite dimensions.
  \item Furthermore, obtaining the gradient density using our framework in the finite $N$ sample setting is simple, efficient, and computable in $O(N \log N)$ time as elucidated in the last paragraph of Section~\ref{sec:charfuncformulation}.
\end{itemize}

\section{Existence of Joint Densities of Smooth Function Gradients}
\indent We begin with a compact measurable subset $\Omega$ of $\mathbb{R}^d$
on which we consider a smooth function $S: \Omega \rightarrow \mathbb{R}$. We
assume that the boundary of $\Omega$ is smooth and the function $S$ is
well-behaved on the boundary as elucidated in
Appendix~\ref{sec:secondkindpoints}. Let $\mathcal{H}_{\bx}$ denote the
\emph{Hessian} of $S$ at a location $\bx \in \Omega$ and let
$\det\left(\mathcal{H}_{\bx}\right)$ denote its determinant. The
\emph{signature} of the Hessian of $S$ at $\bx$, defined as the difference
between the number of positive and negative eigenvalues of $\mathcal{H}_{\bx}$,
is represented by $\sigma_{\bx}$. In order to exactly determine the set of
locations where the joint density of the gradient of $S$ exists, consider the
following three sets:
\begin{align}
\label{eq:setA}
\Au &= \{ \bx : \nabla S(\bx) = \bu\} \\
\label{eq:setB}
\mathcal{B} &= \{ \bx : \; \det\left(\mathcal{H}_{\bx}\right) = 0\} \\
\label{eq:setC}
\mathcal{C} &= \{ \nabla S(\bx) : \bx \in \mathcal{B} \cup \partial \Omega \}.
\end{align}

Let $N(\bu) = |\Au|$. We employ a number of useful lemma, stated here and proved in Appendix~\ref{sec:Proof-of-Lemmas}. 
\\
\begin{lemma}{[}Finiteness Lemma{]}
\label{lemma:finitenessLemma}
$\mathcal{A}_{\bu}$ is finite for every ${\bu} \notin \mathcal{C}$.\\
\end{lemma}

As we see from Lemma~\ref{lemma:finitenessLemma} above, for a given
$\bu\notin\mathcal{C}$, there is only a \emph{finite} collection of $\bx \in
\Omega$ that maps to $\bu$ under the function $\nabla S$. The inverse map
$\nabla S^{(-1)}(\bu)$ which identifies the set of $\bx \in \Omega$ that maps
to $\bu$ under $\nabla S$ is ill-defined as a function as it is a one to many
mapping. The objective of the following lemma (Lemma~\ref{lemma:neighborhoodLemma}) is
to define appropriate neighborhoods such that the inverse function $\nabla
S^{(-1)}$---required in the proof of our main
Theorem~\ref{thm:GradDensity}---when restricted to those neighborhoods is
well-defined.

\begin{lemma}{[}Neighborhood Lemma{]}
\label{lemma:neighborhoodLemma} For every
$\bu_0\notin\mathcal{C}$, there exists a closed neighborhood
$\overline{\mathcal{N}_{\eta}(\bu_0)}$ around $\bu_0$ such that
$\overline{\mathcal{N}_{\eta}(\bu_0)}\cap\mathcal{C}$ is empty. Furthermore,
if $|\mathcal{A}_{\bu_{0}}| >0$, $\overline{\mathcal{N}_{\eta}(\bu_0)}$ can be chosen such
that we can find a closed neighborhood $\overline{\mathcal{N}_{\eta}(\bx)}$
around each $\bx \in \mathcal{A}_{\bu_{0}}$ satisfying the following conditions:
\begin{enumerate}
\item $\nabla S\left(\overline{\mathcal{N}_{\eta}(\bx)}\right) = \overline{\mathcal{N}_{\eta}(\bu_0)}$.
\item $\det\left(\mathcal{H}_{\by}\right) \not=0, \forall \by \in \overline{\mathcal{N}_{\eta}(\bx)}$.
\item The inverse function $\nabla S_{\bx}^{(-1)}(\bu):\overline{\mathcal{N}_{\eta}(\bu_0)} \rightarrow \overline{\mathcal{N}_{\eta}(\bx)}$ is well-defined.
\item For $\by, \bz \in \overline{\mathcal{N}_{\eta}(\bx)}, \sigma_{\by} = \sigma_{\bz}$.\\
\end{enumerate}
\end{lemma}
\begin{lemma} {[}Density Lemma{]}
\label{lemma:densityLemma}
Given $X \sim UNI(\Omega)$, the probability density of $Y = \nabla S(X)$ on $\mathbb{R}^d - \mathcal{C}$ is given by
\begin{equation}
\label{eq:graddensity}
P(\bu) = \frac{1}{\mu(\Omega)} \sum\limits_{k=1}^{N(\bu)} \frac{1}{\left|\det\left(\mathcal{H}_{\bx_k}\right)\right|}
\end{equation}
where $\bx_k \in \Au, \forall k \in \{1,2,\ldots,N(\bu)\}$ and $\mu$ is the Lebesgue measure.\\
\end{lemma}

From Lemma~\ref{lemma:densityLemma}, it is clear that the existence of the
density function $P$ at a location $\bu \in \mathbb{R}^d$ necessitates a
non-vanishing Hessian matrix $(\det(\mathcal{H})\not=0)$ $\forall \bx \in
\Au$. Since we are interested in the case where the density exists almost
everywhere on $\mathbb{R}^d$, we impose the constraint that the set
$\mathcal{B}$ in (\ref{eq:setB}), comprising all points where the Hessian
vanishes, has a Lebesgue measure zero.  It follows that
$\mu(\mathcal{C})=0$. Furthermore, the requirement on the smoothness of $S$
($S \in C^{\infty} (\Omega)$) can be relaxed to functions $S$ in
$C^{\frac{d}{2}+1}(\Omega)$ where $d$ is the dimensionality of $\Omega$ as we will see
in Section~\ref{sec:FiniteDimensionProof}.

\section{Equivalence of the Densities of Gradients and the Power Spectrum}
\label{sec:equivalence}
\indent Define the function $F_\tau : \mathbb{R}^d \rightarrow \mathbb{C}$ as 
\begin{align}\label{eq:Ftau}
F_\tau(\bu) = \frac{1}{(2\pi \tau)^{\frac{d}{2}} \mu(\Omega)^{\frac{1}{2}}} \int_{\Omega}
\exp\left(\frac{i }{\tau} \left[S(\bx) - \bu \cdot \bx\right]\right) d\bx
\end{align}
for $\tau > 0$. $F_\tau$ is very similar to the Fourier transform of the
function $\exp\left(\frac{iS(\bx)}{\tau}\right)$. The normalizing factor in
$F_\tau$ comes from the following lemma (Lemma~\ref{lemma:integralLemma}) whose proof
is given in Appendix~\ref{sec:Proof-of-Lemmas}.
\begin{lemma}{[}Integral Lemma {]}\label{lemma:integralLemma}
$F_\tau \in L^2 (\mathbb{R}^d)$ and $\|F_\tau\|_2 = 1$.\\
\end{lemma}
The power spectrum is defined as \cite{Bracewell99}
\begin{align}
\label{eq:powerspectrum}
P_\tau (\bu) \equiv F_\tau(\bu) \overline{F_\tau(\bu)}.
\end{align} 
Note that $P_{\tau}\geq0$. From Lemma~(\ref{lemma:integralLemma}), we see that
$\int P_{\tau}(\bu)d\bu=1$. Our fundamental contribution lies in interpreting
$P_{\tau}(\bu)$ as a density function and showing its equivalence to the
density function $P(\bu)$ defined in (\ref{eq:graddensity}). Formally stated:
\begin{theorem} 
\label{thm:GradDensity} 
For $\bu_0 \notin \mathcal{C}$,
\begin{equation}
\lim\limits_{\alpha \rightarrow 0} \frac{1}{\mu\left(\mathcal{N}_{\alpha}(\bu_0)\right)}
\lim \limits_{\tau \rightarrow 0} \int\limits_{\mathcal{N}_{\alpha}(\bu_0)}P_\tau (\bu) d\bu = P(\bu_0)
\end{equation}
where $\mathcal{N}_{\alpha}(\bu_0)$ is a ball around $\bu_0$ of radius $\alpha$.\\
 \end{theorem}
Before embarking on the proof, we would like to emphasize
that the ordering of the limits and the integral as given in the theorem
statement is crucial and cannot be arbitrarily interchanged. To press
this point home, we show below that after solving for $P_{\tau}$,
the $\lim_{\tau\rightarrow0}P_{\tau}$ does not exist. Hence, the
order of the integral followed by the limit $\tau\rightarrow0$ cannot
be interchanged. Furthermore, when we swap the limits of $\alpha$
and $\tau$, we get 
\begin{equation}
\lim\limits_{\tau\rightarrow0}\lim\limits_{\alpha\rightarrow0}\frac{1}{\mu\left(\mathcal{N}_{\alpha}(\bu_0)\right)}\int
\limits_{\mathcal{N}_{\alpha}(\bu_0)}P_{\tau}(\bu)d\bu=\lim\limits_{\tau\rightarrow0}P_{\tau}(\bu_{0})
\end{equation}
which also does not exist. Hence, the theorem statement is valid \emph{only}
for the specified sequence of limits and the integral.

\subsection{Brief exposition of the result}
\noindent To understand the result in simpler terms, let us reconsider the
definition of the scaled Fourier transform given in (\ref{eq:Ftau}).  The
first exponential $\exp\left(\frac{iS(\bx)}{\tau}\right)$ is a varying complex
{}``sinusoid\textquotedblright{}, whereas the second exponential
$\exp\left(-\frac{i\bu \cdot \bx}{\tau}\right)$ is a fixed complex sinusoid at
frequency $\frac{\bu}{\tau}$. When we multiply these two complex exponentials,
at low values of $\tau$, the two sinusoids are usually not {}``in
sync\textquotedblright{} and tend to cancel each other out. However, around
the locations where $\nabla S(\bx)=\bu$, the two sinusoids are in perfect sync
(as the combined exponent is \emph{stationary}) with the approximate duration
of this resonance depending on $\det\left(\mathcal{H}_{\bx}\right)$.  The
value of the integral in (\ref{eq:Ftau}) can be increasingly closely
approximated via the stationary phase approximation \cite{Wong89} as
\begin{equation}
F_{\tau}(\bu)\approx\frac{1}{\mu(\Omega)^{\frac{1}{2}}}\sum_{k=1}^{N(u)}
\frac{1}{\sqrt{\left|\det(\mathcal{H}_{\bx_k})\right|}} \exp\left(
\frac{i}{\tau}\left[S(\bx_{k})-\bu \cdot \bx_{k}\right] +i \sigma_{\bx_k}
\frac{\pi}{4} \right).
\end{equation}
The approximation is increasingly tight as $\tau\rightarrow0$. The power
spectrum ($P_{\tau}$) gives us the required result
$\frac{1}{\mu(\Omega)}\sum_{k=1}^{N(u)}\frac{1}{|\det(\mathcal{H}_{\bx_k})|}$
except for the \emph{cross} phase factors $S(\bx_k)-S(\bx_l)-\bu \cdot
(\bx_{k}-\bx_{l})$ obtained as a byproduct of two or more remote locations
$\bx_{k}$ and $\bx_{l}$ indexing into the same frequency bin $\bu$, i.e,
$\bx_{k}\not=\bx_{l}$, but $\nabla S(\bx_{k})=\nabla S(\bx_{l})=\bu$. The
cross phase factors when evaluated are equivalent to
$\cos\left(\frac{1}{\tau}\right)$, the limit of which does not exist as $\tau
\rightarrow 0$. However, integrating the power spectrum over a small
neighborhood $\mathcal{N}_{\alpha}(\bu)$ around $\bu$ removes these cross
phase factors as $\tau$ tends to zero and we obtain the desired result.

\subsection{Formal Proof of Theorem~\ref{thm:GradDensity}}
\indent We wish to compute the integral 
\begin{align}
\label{eq:FTransform}
 F_{\tau}(\bu) = \frac{1}{(2 \pi \tau )^{\frac{d}{2}}
   \mu(\Omega)^{\frac{1}{2}}} \int\limits_{\Omega} \exp\left\{\frac{i}{\tau}
 (S(\bx) - \bu \cdot \bx)\right\} d\bx
\end{align}
at small values of $\tau$ and exhibit the connection between the power
spectrum $P_{\tau}(\bu)$ and the density function $P(\bu)$. To this end define
$\Psi(\bx;\bu) \equiv S(\bx) - \bu \cdot \bx$. The proof follows by
considering two cases: the first case in which there are no stationary points
and therefore the density should go to zero, and the second case in which stationary
points exist and the contribution from the oscillatory integral is shown to
increasingly closely approximate the density function of the gradient as $\tau
\rightarrow 0$.\\

\noindent \textbf{case (i)}: We first consider the case where $N(\bu)=0$, i.e,
$\bu \notin \nabla S(\Omega)$. In other words there are no stationary points
\cite{Wong89} for this value of $\bu$. The proof that this case yields the
anticipated contribution of zero follows clearly from a straightforward
technique commonly used in stationary phase expansions. We assume that the
function $S$ is sufficiently well-behaved on the boundary such that the total
contribution due to the stationary points of the second kind \cite{Wong89}
approaches zero as $\tau \rightarrow 0$. (Concentrating here on the crux of our
work, we reserve the discussion concerning the behavior of $S$ on the boundary
and the relationship to stationary points of the second kind to
Appendix~\ref{sec:secondkindpoints}.) Under mild conditions (outlined in
Appendix~\ref{sec:secondkindpoints}), the contributions from the stationary
points of the third kind can also be ignored as they approach zero as $\tau$
tends to zero \cite{Wong89}. Higher order terms follow suit.\\
\begin{lemma}
\label{lemma:emptylemma}
Fix $\bu_0 \notin \mathcal{C}$. If $\mathcal{A}_{\bu_0} = \emptyset$ then
$F_{\tau}(\bu_0) = O(\sqrt{\tau})$ as $\tau \rightarrow 0$.\\
\end{lemma}
\begin{proof}
To improve readability, we prove Lemma~\ref{lemma:emptylemma} first in the one
dimensional setting and separately offer the proof for multiple dimensions.

\subsubsection{Proof of Lemma~\ref{lemma:emptylemma} in one Dimension}
Let $s$ denote the derivative (1D gradient) of $S$, i.e, $s(x) := S^{\prime}(x)$.
The bounded closed interval $\Omega$ is represented by $\Omega=[b_{1},b_{2}]$,
with the length $L=\mu(\Omega)=b_{2}-b_{1}$. As $u_0 \notin \mathcal{C}$,
there is no $x\in\Omega$ for which $s(x)=u_0$. Recalling the definition of
$\Psi$, namely $\Psi(x;u) \equiv S(x) - u x$, we see that $\Psi^{\prime}(x)
\not=0$ and is of constant sign in $[b_{1},b_{2}]$. It follows that $\Psi(x)$
is strictly monotonic. Defining $v=\Psi(x)$, we have from (\ref{eq:Ftau})
\begin{equation}
F_{\tau}(u_{0})=\frac{1}{\sqrt{2\pi\tau L}}
\int_{\Psi(b_{1})}^{\Psi(b_{2})}\exp\left(\frac{iv}{\tau}\right)t(v)dv.
\end{equation}
 Here $t(v)=\frac{1}{\Psi^{\prime}\left( \Psi^{-1}(v) \right)}$. The inverse
 function is guaranteed to exist due to the monotonicity of
 $\Psi$. Integrating by parts we get
\begin{eqnarray}
F_{\tau}(u_{0})\sqrt{2\pi\tau L} & = &
\frac{\tau}{i}\left[\exp\left(\frac{i\Psi(b_{2})}{\tau}\right)
  t\left(\Psi(b_{2})\right)-
  \exp\left(\frac{i\Psi(b_{1})}{\tau}\right)t\left(\Psi(b_{1})\right)\right]\nonumber
\\ & &
-\frac{\tau}{i}\int_{\Psi(b_{1})}^{\Psi(b_{2})}\exp\left(\frac{iv}{\tau}\right)t^{\prime}(v)dv.\label{eq:ibp}
 \end{eqnarray}
It follows that
 \begin{equation}
\left|F_{\tau}(u_{0})\right| \leq \frac{\sqrt{\tau}} {\sqrt{2\pi L}}
\left(\frac{1}{|s(b_{2})-u_{0}|}+\frac{1}{|s(b_{1})-u_{0}|}
+\int_{\Psi(b_{1})}^{\Psi(b_{2})}\left|t^{\prime}(v)\right|dv\right).
\end{equation}
 
\subsubsection{Proof of Lemma~\ref{lemma:emptylemma} in Finite Dimensions}
\label{sec:FiniteDimensionProof} 
As $\nabla \Psi(\bx;\bu_0) \not=\mathbf{0}, \forall \bx$, the vector field 
 \begin{equation}
 \bv_1(\bx) = \frac{\nabla \Psi}{\|\nabla \Psi\|^2}
 \end{equation}
 is well-defined. Choose $m>\frac{d}{2}$ (with this choice explained below) and
 for $j \in \{1,2,\ldots,m\}$, recursively define the function $g_{j}(\bx)$ and
 the vector field $ \bv_{j+1}(\bx)$ as follows:
  \begin{align}
  g_{1}(\bx)  &= 1,\\
  g_{j+1}(\bx)  &= \nabla \cdot \bv_j (\bx) \mbox{   and}\\
 \bv_{j+1}(\bx)  &= \frac{\nabla \Psi}{\|\nabla \Psi\|^2}g_{j+1}(\bx).
 \end{align}
 Using the equality
 \begin{align}
\label{eq:exppsirewrite}
\exp\left(\frac{i}{\tau}\Psi(\bx;\bu_0)\right) g_{j}(\bx) &= i\tau
\left[\nabla \cdot \bv_j(\bx) \right] \exp\left(\frac{i}{\tau}\Psi(\bx;\bu_0)
\right) \nonumber \\ &- i\tau\nabla \cdot \left[\bv_j(\bx)
  \exp\left(\frac{i}{\tau}\Psi(\bx;\bu_0) \right)\right]
\end{align}
where $\nabla \cdot$ is the divergence operator, and applying the divergence
theorem $m$ times, the Fourier transform in (\ref{eq:FTransform}) can be
rewritten as
 \begin{align}
F_{\tau}(\bu_0) &= \frac{1}{ (2 \pi \tau )^{\frac{d}{2}} \mu(\Omega)^{\frac{1}{2}}}
\left(i \tau\right)^m \int\limits_{\Omega} g_{m+1}(\bx)
\exp\left(\frac{i}{\tau } \Psi(\bx)\right) d\bx \nonumber \\
\label{eq:FTrewritten}
 &- \frac{1}{ (2 \pi \tau )^{\frac{d}{2}} \mu(\Omega)^{\frac{1}{2}}} \sum\limits_{j
  =1}^{m} (i \tau)^{j} \int\limits_{\partial \Omega} \left[\bv_j(\bx(\by))
  \cdot \mathbf{n}\right] \exp\left(\frac{i}{\tau } \Psi(\bx(\by))\right)d\by
 \end{align}
 which is similar to \eqref{eq:ibp}.

 We would like to add a note on the differentiability of $S$ which we briefly
 mentioned after Lemma~\ref{lemma:densityLemma}. The divergence theorem is
 applied $m>\frac{d}{2}$ times to obtain sufficiently higher order powers of
 $\tau$ in the numerator so as to exceed the $\tau^{\frac{d}{2}}$ term in the
 denominator of the first line of (\ref{eq:FTrewritten}). This necessitates
 that $S$ is at least $\frac{d}{2}+1$ times differentiable. The smoothness
 constraint on $S$ can thus be relaxed and replaced by $S \in
 C^{\frac{d}{2}+1} (\Omega)$.
 
The additional complication of the $d$-dimensional proof lies in resolving the
geometry of the terms in the second line of (\ref{eq:FTrewritten}). Here
$\mathbf{n}$ is the unit outward normal to the positively oriented boundary
$\partial \Omega$ parameterized by $\by$. As $m>\frac{d}{2}$, the term on the right
side of the first line in (\ref{eq:FTrewritten}) is $o(\sqrt{\tau})$ and hence
can be overlooked. To evaluate the remaining integrals within the summation in
(\ref{eq:FTrewritten}), we should take note that the stationary points of the
second kind for $\Psi(\bx)$ on $\Omega$ correspond to the first kind of
stationary points for $\Psi(\bx(\by))$ on the boundary $\partial \Omega$. We
show in case (ii) that the contribution of a stationary point of the first
kind in a $d-1$ dimensional space is $O(\tau^{\frac{d-1}{2}})$. As the
dimension of $\partial \Omega$ is $d-1$, we can conclude that the total
contribution from the stationary points of the second kind is
$O(\tau^{\frac{d-1}{2}})$. After multiplying and dividing this contribution by
the corresponding $\tau^{j}$ and $\tau^{\frac{d}{2}}$ factors respectively, it is easy
to see that the contribution of the $j^{\mathrm{th}}$ integral (out of the $n$
integrals in the summation) in (\ref{eq:FTrewritten}) is
$O\left(\tau^{j-\frac{1}{2}}\right)$ and hence the total contribution of the $m$
integrals is of $O(\sqrt{\tau})$. Here we have safely ignored the
stationary points of the third kind as their contributions are minuscule
compared to the other two kinds as shown in \cite{Wong89}. Combining all the
terms in (\ref{eq:FTrewritten}) we get the desired result, namely
$F_{\tau}(u_0) = O(\sqrt{\tau})$. For a detailed exposition of the proof, we
encourage the reader to refer to Chapter 9 in \cite{Wong89}.\\
\end{proof}

We then get $P_{\tau}(\bu_0) = O(\tau)$. Since $ \nabla S(\Omega)$ is a
compact set in $\mathbb{R}^d$ and $\bu_0 \notin \nabla S(\Omega)$, we can
choose a neighborhood $\Neta(\bu_0)$ around $\bu_0$ such that for $\bu \in
\Neta(\bu_0)$, no stationary points exist and hence
\begin{align}
\lim\limits_{\tau \rightarrow 0} \int\limits_{\Neta(\bu_0)} P_{\tau}(\bu) d\bu =0.
\end{align}
Since the cardinality $N(\bu)$ of the set $\Au$ defined in (\ref{eq:setA}) is
zero for $\bu \in \Neta(\bu_0)$, the true density $P(\bu)$ of the random
variable transformation $Y=\nabla S(X)$ given in (\ref{eq:graddensity}) also
vanishes for $\bu \in \Neta(\bu_0)$.\\\\

\noindent \textbf{case (ii)}: For $\bu_0 \notin \mathcal{C}$, let $N(\bu_0) >
0$. In this case, the number of stationary points in the \emph{interior} of
$\Omega$ is non-zero and \emph{finite} as a consequence of
Lemma~\ref{lemma:finitenessLemma}. We can then rewrite
\begin{align}
\label{eq:Ftau2}
F_{\tau}(\bu_0)= G + \frac{1}{(2 \pi \tau )^{\frac{d}{2}} \mu(\Omega)^{\frac{1}{2}}}
\sum\limits_{k=1}^{N(\bu_0)} \int\limits_{\overline{\Neta(\bx_k)}}
\exp\left(\frac{i}{\tau} \Psi(\bx;\bu_0)\right) d\bx
\end{align}
where
\begin{align}
\label{eq:Gempty}
G \equiv \int\limits_{K} \exp\left(\frac{i}{\tau} \Psi(\bx;\bu_0) \right) d\bx.
\end{align}
and the domain $K \equiv \Omega \setminus \bigcup\limits_{k=1}^{N(\bu_0)}
\overline{\Neta(\bx_k)}$. The closed regions
$\left\{\overline{\Neta(\bx_k)}\right\}_{k=1}^{N(\bu_0)}$ are obtained from
Lemma~\ref{lemma:neighborhoodLemma}.

Firstly, note that the the set $K$ contains no stationary points by
construction. Secondly, the boundaries of $K$ can be classified into two
categories: those that overlap with the sets $\overline{\Neta(\bx_k)}$ and
those that coincide with $\Gamma = \partial \Omega$. Propitiously, the
orientation of the overlapping boundaries between the sets $K$ and each
$\overline{\Neta(\bx_k)}$ are in opposite directions as these sets are located
at different sides when viewed from the boundary. Hence, we can \emph{exclude}
the contributions from the overlapping boundaries between $K$ and
$\overline{\Neta(\bx_k)}$ while evaluating $F_{\tau}(\bu_0)$ in
(\ref{eq:Ftau2}) as they \emph{cancel} each other out.

To compute $G$ we leverage case (i) which also includes the contribution from the boundary $\Gamma$ and get 
\begin{equation}
\label{eq:epsilon1}
G = \epsilon_1(\bu_0,\tau) = O(\sqrt{\tau}).
\end{equation}

To evaluate the remaining integrals over $\overline{\Neta(\bx_k)}$, we take
into account the contribution from the stationary point at $\bx_k$ and obtain
\begin{align}
\label{eq:statpha}
\int\limits_{\overline{\Neta(\bx_k)}} \exp\left(\frac{i}{\tau}
\Psi(\bx;\bu_0)\right) d\bx = &\frac{a(2\pi\tau)^{\frac{d}{2}}}{ \sqrt{|\det
    (\mathcal{H}_{\bx_k})|}} \exp\left(\frac{i}{\tau} \Psi(\bx_k;\bu_0) + i
\sigma_{\bx_k} \frac{\pi}{4} \right) \nonumber \\ &+ \epsilon_2(\bu_0,\tau)
\end{align}
where 
\begin{equation}
\label{eq:epsilon2}
\epsilon_2(\bu_0,\tau)=O\left(\tau^{\frac{d+1}{2}}\right) \leq \tau^{\frac{d+1}{2}} \gamma_{1}(\bu_{0})
\end{equation}
for a continuous bounded function $\gamma_{1}(\bu)$ \cite{Wong89}. The
variable $a$ in (\ref{eq:statpha}) takes the value $1$ if $\bx_k$ lies in the
interior of $\Omega$, otherwise equals $\frac{1}{2}$ if $\bx_k \in \partial
\Omega$. Since $\bu \notin \mathcal{C}$, stationary points do not occur on
the boundary and hence $a=1$ for our case. Recall that $\sigma_{\bx_k}$ is the
signature of the Hessian at $\bx_k$. Note that the main term has the factor
$\tau^{\frac{d}{2}}$ in the numerator when we perform stationary phase in $d$ dimensions as
we mentioned under the finite dimensional proof of
Lemma~\ref{lemma:emptylemma}.

Coupling (\ref{eq:Ftau2}), (\ref{eq:Gempty}), and (\ref{eq:statpha}) yields
\begin{align}
\label{eq:Ftau_approx}
F_{\tau}(\bu_0) = \frac{1}{\mu(\Omega)^{\frac{1}{2}}}
\sum\limits_{k=1}^{N(\bu_0)}\exp\left(\frac{i}{\tau} \Psi(\bx_k;\bu_0) + i
\sigma_{\bx_k} \frac{\pi}{4} \right) \frac{1}{\sqrt{|\det
    (\mathcal{H}_{\bx_k})|}} +\epsilon_3(\bu_0,\tau)
\end{align}
where
\begin{equation}
\epsilon_3(\bu_0,\tau) = \epsilon_1(\bu_0,\tau) + \frac{\epsilon_2(\bu_0,\tau)}{(2 \pi \tau )^{\frac{d}{2}}\mu(\Omega)^{\frac{1}{2}}}.
\end{equation}
As $\epsilon_1(\bu_0,\tau) = O(\sqrt{\tau})$ and $\epsilon_2(\bu_0,\tau) =
O\left(\tau^{\frac{d+1}{2}}\right)$ from (\ref{eq:epsilon1}) and
(\ref{eq:epsilon2}) respectively, we have $\epsilon_3(\bu_0,\tau) =
O(\sqrt{\tau})$.  Based on the definition of the power spectrum $P_{\tau}$ in
(\ref{eq:powerspectrum}), we get
\begin{align}
P_\tau (\bu_0) &= \frac{1}{\mu(\Omega)} \sum\limits_{k=1}^{N(\bu_0)}
\frac{1}{|\det (\mathcal{H}_{\bx_k})|} \nonumber \\ &+\frac{1}{\mu(\Omega)}
\sum\limits_{k=1}^{N(\bu_0)} \sum\limits_{\substack{l=1\\l\neq k}}^{N(\bu_0)}
\frac{\exp\left\{\frac{i}{\tau}
  \left[\Psi(\bx_k;\bu_0)-\Psi(\bx_l;\bu_0)\right]+ i (\sigma_{\bx_k} -
  \sigma_{\bx_l})\frac{\pi }{ 4}\right\}}{\sqrt{|\det(\mathcal{H}_{\bx_k})|}
  \sqrt{|\det (\mathcal{H}_{\bx_l})|}}\nonumber \\ &+
\epsilon_{4}(\bu_{0},\tau)
\label{PRIMARY}
\end{align}
 where $\epsilon_{4}(\bu_{0})$ includes both the squared magnitude of
 $\epsilon_{3}(\bu_{0},\tau)$ and the cross terms involving the first term in
 (\ref{eq:Ftau_approx}) and $\epsilon_{3}(\bu_{0},\tau)$. Notice that the main
 term in (\ref{eq:Ftau_approx}) can be bounded \emph{independently} of $\tau$
 as
\begin{equation}
\left|\exp\left(\frac{i}{\tau}\Psi(\bx_k;\bu_0) + i \sigma_{\bx_k} \frac{\pi}{4} \right)\right|=1,\forall\tau \neq 0
\end{equation}
and $\det(\mathcal{H}_{\bx_k})\not=0,\forall k$. Since
$\epsilon_{3}(\bu_{0},\tau)=O(\sqrt{\tau})$, we get
$\epsilon_{4}(\bu_{0},\tau)=O(\sqrt{\tau})$. Furthermore, as
$\epsilon_{4}(\bu_{0},\tau)$ can be also be uniformly bounded by a function of
$\bu$ for small values of $\tau$, we have
\begin{equation}
\label{eq:errortermequalzero}
\lim\limits_{\tau \rightarrow 0} \int\limits_{\overline{\Neta(\bu_0)}}\epsilon_{4}(\bu_{0},\tau) d\bu=0.
\end{equation}
Observe that the term on the right side of the first line in (\ref{PRIMARY})
matches the anticipated expression for the density function $P(u_{0})$ given
in (\ref{eq:graddensity}). The cross phase factors in the second line of
(\ref{PRIMARY}) arise due to multiple remote locations $\bx_k$ and $\bx_l$
indexing into $\bu$. The cross phase factors when evaluated can be shown to be
proportional to $\cos\left(\frac{1}{\tau}\right)$. Since
$\lim_{\tau\rightarrow0}\cos\left(\frac{1}{\tau}\right)$ is not
defined, $\lim_{\tau\rightarrow0}P_{\tau}(\bu_{0})$ does not exist. We briefly
alluded to this problem immediately following the statement of
Theorem~\ref{thm:GradDensity} in Section~\ref{sec:equivalence}. However, the
following lemma which invokes the inverse function $\nabla
S_{\bx}^{(-1)}(\bu):\overline{\Neta(\bu_0)} \rightarrow
\overline{\Neta(\bx)}$---defined in Lemma~\ref{lemma:neighborhoodLemma} where
$\bx$ is written as a function of $\bu$---provides a simple way to nullify the
cross phase factors. Note that since each $\nabla S_{\bx}^{(-1)}$ is a
bijection, $N(\bu)$ doesn't vary over $\overline{\Neta(\bu_0)}$. Pursuant to
Lemma~\ref{lemma:neighborhoodLemma}, the Hessian signatures
$\sigma_{\bx_k(\bu)}$ and $\sigma_{\bx_l(\bu)}$ also remain constant over
$\overline{\Neta(\bu_0)}$.  \\

\begin{lemma}{[}Cross Factor Nullifier Lemma{]}
\label{lemma:cross}
The integral of the cross term in the second line of (\ref{PRIMARY}) over the
closed region $\overline{\Neta(\bu_0)}$ approaches zero as $\tau \rightarrow
0$, i.e, $\forall k \not=l$
\begin{equation}
\label{eq:crosstermequalzero}
\lim\limits_{\tau \rightarrow 0} \int\limits_{\overline{\Neta(\bu_0)}}
\frac{\exp\left\{\frac{i}{\tau}
  \left[\Psi(\bx_k(\bu);\bu)-\Psi(\bx_l(\bu);\bu)\right]
  \right\}}{|\det(\mathcal{H}_{\bx_k(\bu)})|^{\frac{1}{2}}|\det
  (\mathcal{H}_{\bx_l(\bu)})|^{\frac{1}{2}}} d\bu = 0.
\end{equation}
\end{lemma}
The proof is given in Appendix~\ref{sec:Proof-of-Lemmas}. Combining
(\ref{eq:errortermequalzero}) and (\ref{eq:crosstermequalzero}) yields
\begin{equation}
\label{eq:cumulativeequivalence}
\lim\limits_{\tau \rightarrow 0} \int\limits_{\overline{\Neta(\bu_0)}} P_\tau
(\bu) d\bu = \frac{1}{\mu(\Omega)} \int\limits_{\overline{\Neta(\bu_0)}}
\sum\limits_{k=1}^{N(\bu)} \frac{1}{|\det (\mathcal{H}_{\bx_k(\bu)})|} d\bu =
\int\limits_{\overline{\Neta(\bu_0)}} P(\bu) d\bu.
\end{equation}
Equation~\ref{eq:cumulativeequivalence} demonstrates the equivalence of the
cumulative distributions corresponding to the densities $P_\tau (\bu)$ and
$P(\bu)$ when integrated over \emph{any} sufficiently small neighborhood
$\overline{\Neta(\bu_0)}$ of radius $\eta$. To recover the density $P(\bu)$,
we let $\alpha < \eta$ and take the limit with respect to $\alpha$ to complete
the proof.

Taking a mild digression from the main theme of this paper, in the next
section (Section~\ref{sec:charfuncformulation}), we build an informal bridge between the
commonly used characteristic function formulation for computing densities and
our wave function method. The motivation behind this section is merely to
provide an intuitive reason behind our main theorem (Theorem~\ref{thm:GradDensity})
where we directly manipulate the power spectrum of
$\phi(\bx)=\exp\left(\frac{iS(\bx)}{\tau}\right)$ into the characteristic
function formulation stated in (\ref{eq:iftcharfunc}), circumventing the need
for the closed-form expression of the density function $P(\bu)$ given in
(\ref{eq:graddensity}). We request the reader to bear in mind the following
cautionary note. What we show below \emph{cannot} be treated as a formal
proof of Theorem~\ref{thm:GradDensity}. Our attempt here merely provides a
mathematically intuitive justification for establishing the equivalence between the
power spectrum and the characteristic function formulations and thereby to the
density function $P(\bu)$. On the basis of the reasons described therein, we
strongly believe that the mechanism of stationary phase is essential to
formally prove our main theorem (Theorem~\ref{thm:GradDensity}). It is best to treat the
wave function and the characteristic function methods as two \emph{different}
approaches for estimating the probability density functions and \emph{not}
reformulations of each other. To press this point home, we also comment on the
computational complexity of the wave function and the characteristic
function methods at the end of the next section.

\section{Relation between the characteristic function and the power spectrum
  formulations of the gradient density}
\label{sec:charfuncformulation}
The characteristic function $\psi_Y(\bomega)$ for the random variable
$Y=\nabla S(X)$ is defined as the expected value of $ \exp\left(i \bomega
\cdot Y\right)$, namely
\begin{equation}
\psi_Y(\bomega) \equiv E\left[ \exp\left(i \bomega \cdot Y \right) \right] =
\frac{1}{\mu(\Omega)} \int_{\Omega} \exp\left(i \bomega \cdot \nabla
S(\bx)\right) d \bx.
\label{eq:characfunc}
\end{equation} 
Here $ \frac{1}{\mu(\Omega)}$ denotes the density of the uniformly distributed random variable $X$ on $\Omega$.

The inverse Fourier transform of a characteristic function also serves as the
density function of the random variable under consideration \cite{Billingsley95}. In other words,
the density function $P(\bu)$ of the random variable $Y$ can be obtained via
\begin{align}
\label{eq:iftcharfunc}
P(\bu) &= \frac{1}{(2\pi)^d} \int \psi_Y(\omega) \exp\left(-i \bomega \cdot \bu\right) d\bomega \nonumber \\
&= \frac{1}{(2\pi)^d \mu(\Omega)}\int \int_{\Omega} \exp\left(i \bomega \cdot \left[\nabla S(\bx)-\bu\right]\right) d\bx d\bomega.
\end{align}

Having set the stage, we can now proceed to highlight the close relationship
between the characteristic function formulation of the density and our
formulation arising from the power spectrum. For simplicity, we choose to
consider a region $\Omega$ that is the product of closed intervals, $\Omega =
\prod\limits_{i=1}^d [a_i,b_i ]$. Based on the expression for the scaled
Fourier transform $F_{\tau}(\bu)$ in (\ref{eq:Ftau}), the power spectrum
$P_{\tau}(\bu)$ is given by
\begin{equation}
P_\tau (\bu)=\frac{1}{(2\pi \tau )^d \mu(\Omega)} \int_{\Omega} \int_{\Omega}
\exp\left\{\frac{i }{\tau} \left[S(\bx) - S(\by)-(\bu \cdot
  (\bx-\by))\right]\right\} d\bx d\by.
\end{equation}
Define the following change of variables, 
\begin{equation}
\bomega = \frac{\bx - \by}{\tau},  \bnu = \frac{\bx + \by}{2}.
\end{equation}
Then,
\begin{equation}
 \bx = \bnu + \frac{\bomega \tau}{2}, \by = \bnu - \frac{\bomega \tau}{2}
\end{equation}
and the integral limits for $\bomega$ and $\bnu$ are given by
\begin{align}
W &= \prod\limits_{i=1}^d \left[\frac{a_i-b_i}{\tau},
  \frac{b_i-a_i}{\tau}\right] \\ V_{\bomega} &= \prod\limits_{i=1}^d \left[a_i
  + \frac{|\-\omega_i| \tau}{2}, b_i - \frac{|\-\omega_i|\tau}{2}\right]
\end{align}
where $\omega_i$ is the $i^{\mathrm{th}}$ component of $\bomega$. Note that the
Jacobian of this transformation is $\tau^d$.  Now we may write the integral
$P_\tau(\bu)$ in terms of these new variables as
\begin{equation}
\label{eq:Ptauomega}
P_\tau(\bu) =  \frac{1}{(2\pi)^d \mu(\Omega)} \int\limits_{W } \xi(\bomega, \bu) d\bomega
\end{equation}
where
\begin{equation}
\label{eq:xi}
\xi(\bomega, \bu) = \int\limits_{V_{\bomega}} \exp\left\{\frac{i
}{\tau}\left[S\left(\bnu + \frac{\tau\bomega}{2}\right) - S\left(\bnu -
  \frac{\tau \bomega}{2}\right)\right]\right\} \exp\left(-i \bu \cdot
\bomega\right)d\bnu.
\end{equation}

The mean value theorem applied to $S\left(\bnu + \frac{\tau\bomega}{2}\right)
- S\left(\bnu - \frac{\tau \bomega}{2}\right)$ yields
\begin{align}
\label{eq:meanvaluethm}
\xi(\bomega, \bu) = \int\limits_{V_{\bomega}} \exp\left(i \bomega \cdot \left[
  \nabla S(\bz(\bnu,\bomega))-\bu\right] \right) d\bnu,
\end{align}
where
\begin{equation}
\bz(\bnu,\bomega) \equiv c\left(\bnu + \frac{\tau\bomega}{2}\right) +(1-c)
\left(\bnu - \frac{\tau \bomega}{2}\right)
\end{equation}
with $c \in [0,1]$. When $\bomega$ is fixed and $\tau \rightarrow 0$,
$\bz(\bnu,\bomega) \rightarrow \bnu$ and so for small values of $\tau$ we get
\begin{align}
\label{eq:approxxi}
\xi(\bomega, \bu) \approx \int\limits_{\bnu \in V_{\bomega}} \exp\left(i
\bomega \cdot \left[ \nabla S(\bnu)- \bu \right] \right) d\bnu.
\end{align}
Again we would like to drive the following point home. We \emph{do not} claim
that we have formally proved the above approximation. On the contrary, we
believe that it might be an onerous task to do so as the mean value theorem
point $\bz$ in (\ref{eq:meanvaluethm}) is unknown and the integration limits
for $\nu$ directly depend on $\tau$. The approximation is stated with the sole
purpose of providing an intuitive reason for our theorem (Theorem~\ref{thm:GradDensity})
and to provide a clear link between the characteristic function and wave
function methods for density estimation.

Furthermore, note that the integral range for $\bomega$ depends on $\tau$ and so when
$\bomega = O\left(\frac{1}{\tau}\right)$, $\bomega \tau \not\rightarrow 0$ as
$\tau \rightarrow 0$ and hence the above approximation for $\xi(\bomega, \bu)$
in (\ref{eq:approxxi}) might seem to break down. To evade this seemingly
ominous problem, we manipulate the domain of integration for $\bomega$ as
follows. Fix an $\epsilon \in (0,1)$ and let
\begin{align}
W &= W^{\infty} \cup W^{\tau} = 
\left(W \setminus \prod\limits_{i=1}^d [-M_i, M_i] \right) \cup \prod\limits_{i=1}^d [-M_i,M_i] 
\end{align}
where
\begin{equation}
\label{eq:Mi}
M_i \equiv (b_i-a_i)\tau^{\epsilon - 1}.
\end{equation}
By defining $M_i$ as above, note that in $W^\tau$, $\bomega$ is deliberately
made to be $O(\tau^{\epsilon-1})$ and hence $\bomega \tau \rightarrow 0$ as
$\tau \rightarrow 0$. Hence the approximation for $\xi(\bomega, \bu)$ in
(\ref{eq:approxxi}) might hold for this integral range of $\bomega$. For
consideration of $\bomega \in W^{\infty}$, recall that
Theorem~\ref{thm:GradDensity} requires the power spectrum $P_{\tau}(\bu)$ to
be integrated over a small neighborhood $\mathcal{N}_{\alpha}(\bu_0)$ around
$\bu_0$. By using the true expression for $\xi(\bomega,\bu)$ from
(\ref{eq:xi}) and performing the integral for $\bu$ prior to $\bomega$ and
$\bnu$, we get
\begin{align}
& \int\limits_{\mathcal{N}_{\alpha}(\bu_0)}\int\limits_{W^{\infty}}
  \xi(\bomega, \bu) d\bomega d\bu = \nonumber \\ & \int\limits_{W^{\infty}
  }\int\limits_{V_{\bomega}} \exp\left\{i \bomega \cdot \left[ \nabla S\left(
    c\left(\bnu + \frac{\tau\bomega}{2}\right) +(1-c) \left(\bnu - \frac{\tau
      \bomega}{2}\right)\right)\right] \right\}
  \left[\int\limits_{\mathcal{N}_{\alpha}(\bu_0)} \exp\left(-i \bomega \cdot
    \bu \right)d\bu \right] d\bnu d\bomega.
\end{align}
Since both $M_i$ in (\ref{eq:Mi}) and the lower and the upper limits for
$\omega_i$, namely $\pm \frac{b_i-a_i}{\tau}$ respectively approach $\infty$
as $\tau \rightarrow0$, the Riemann-Lebesgue lemma \cite{Bracewell99}
guarantees that $\forall \bomega \in W^{\infty}$, the integral
\begin{align}
\int\limits_{\mathcal{N}_{\alpha}(\bu_0)} \exp\left(-i \bomega \cdot \bu \right)d\bu
\end{align}
 approaches zero as $\tau \rightarrow 0$. Hence for small values of $\tau$, we
 can expect the integral over $W^{\tau}$ to dominate over the other. This
 leads to the following approximation,
\begin{equation}
\label{eq:Ptauapprox}
 \int\limits_{\mathcal{N}_{\alpha}(\bu_0)} P_{\tau}(\bu) d\bu \approx
 \frac{1}{(2\pi)^d \mu(\Omega)} \int\limits_{\mathcal{N}_{\alpha}(\bu_0)}
 \int\limits_{W^{\tau}} \xi(\bomega,\bu) d\bomega d\bu,
\end{equation}
as $\tau$ approaches zero. Combining the above approximation with the
approximation for $\xi(\bomega,\bu)$ given in (\ref{eq:approxxi}) and noting
that the integral domain for $\bomega$ and $\bnu$ approaches $\mathbb{R}^d$
and $\Omega$ respectively as $\tau \rightarrow 0$, the integral of the power
spectrum $P_\tau(u)$ over the neighborhood $\mathcal{N}_{\alpha}(\bu_0)$ at
small values of $\tau$ in (\ref{eq:Ptauomega}) can be approximated by
\begin{equation}
 \int\limits_{\mathcal{N}_{\alpha}(\bu_0)} P_{\tau}(u) du \approx
 \frac{1}{(2\pi)^d \mu(\Omega)} \int\limits_{\mathcal{N}_{\alpha}(\bu_0)} \int
 \int\limits_{\Omega} \exp\left(i \bomega \cdot \left[ \nabla S(\bnu)- \bu
   \right] \right) d\bnu d\bomega d\bu.
\end{equation}
This form \emph{exactly} coincides with the expression given in
(\ref{eq:iftcharfunc}) obtained through the characteristic function
formulation.

The approximations given in (\ref{eq:approxxi}) and (\ref{eq:Ptauapprox})
cannot be proven easily as they involve limits of integration which directly
depend on $\tau$. Furthermore, the mean value theorem point
$\bz(\bnu,\bomega)$ in (\ref{eq:meanvaluethm}) is arbitrary and \emph{cannot
  be determined beforehand for a given value of $\tau$}. The difficulties
faced here emphasize the need for the method of stationary phase to formally
prove Theorem~\ref{thm:GradDensity}.

As we remarked before, the characteristic function and our wave function methods should \emph{not} be treated as mere reformulations of each other. This distinction is further emphasized when we find our method to be computational efficient  over the characteristic function approach in the finite sample-set scenario where we estimate the gradient density from a finite $N$ samples of the function $S$. Given these $N$ sample
values $\hat{S}$ and its gradient $\nabla \hat{S}$, the characteristic
function defined in (\ref{eq:characfunc}) needs to be computed for $N$
integral values of $\bomega$. Each value of $\bomega$ requires summation over
the $N$ sampled values of $\exp\left(i\bomega \nabla S(\bx)\right)$. Hence the
total time required to determine the characteristic function is $O(N^2)$.  The
joint density function of the gradient is obtained via the inverse Fourier
transform of the the characteristic function, which is an $O(N \log N)$
operation. The overall time complexity is therefore $O(N^2)$. In our wave
function method the Fourier transform of $\exp\left(\frac{i
  \hat{S}(\bx)}{\tau}\right)$ at a given value of $\tau$ can be computed in
$O(N \log N)$ and the subsequent squaring operation to obtain the power
spectrum can be performed in $O(N)$. Hence the density function can be
determined in $O(N \log N)$, which is more efficient when compared to the
$O(N^2)$ complexity of the characteristic function approach.

\section{Discussion}
Observe that the integrals 
\begin{equation}
I_{\tau}(\bu_{0})= \int_{\Neta(\bu_0)}P_\tau (\bu)
d\bu, \hspace{10pt}I(\bu_{0})=\int_{\Neta(\bu_0)} P(\bu) d\bu
\end{equation}
 give the interval measures of the density functions $P_{\tau}$ and $P$
 respectively. Theorem~\ref{thm:GradDensity} states that at small values of
 $\tau$, both the interval measures are approximately equal, with the
 difference between them being $O(\sqrt{\tau})$. Recall that by definition,
 $P_{\tau}$ is the normalized power spectrum of the wave function
 $\phi(\bx)=\exp\left(\frac{iS(\bx)}{\tau}\right)$.  Hence we conclude that
 the power spectrum of $\phi(\bx)$ can potentially serve as a \emph{joint
   density estimator} for the gradient of $S$ at small values of $\tau$.  We
 also built an informal bridge between our wave function method and the
 characteristic function approach for estimating probability densities by
 directly trying to recast the former expression into the latter. The
 difficulties faced in relating the two approaches reinforce the stationary phase
 method as a powerful tool to formally prove Theorem~\ref{thm:GradDensity}.  Our
 earlier result proved in \cite{Gurumoorthy12} where we employ the stationary
 phase method to compute the gradient density of Euclidean distance functions
 in two dimensions, is now generalized in Theorem~\ref{thm:GradDensity} which
 establishes similar gradient density estimation result for arbitrary smooth
 functions in arbitrary finite dimensions.

\appendix
\section{Proof of Lemmas}
\label{sec:Proof-of-Lemmas}

\noindent \emph{1. Proof of Finiteness Lemma}
\begin{proof}
 We prove the result by contradiction. Observe that $\mathcal{A}_{\bu}$ is a
 subset of the compact set $\Omega$. If $\mathcal{A}_{\bu}$ is not finite,
 then by Theorem~(2.37) in \cite{Rudin76}, $\mathcal{A}_{\bu}$ has a limit
 point $\mathbf{x}_0\in\Omega$. If $\mathbf{x}_0 \in \partial \Omega$, then
 $\bu \in \mathcal{C}$ giving a contradiction. Otherwise, consider a sequence
 $\{\mathbf{x}_n\}_{n=1}^{\infty}$, with each
 $\mathbf{x}_n\in\mathcal{A}_{\bu}$, converging to $\mathbf{x}_0$. Since
 $\nabla S(\mathbf{x}_n)=\bu$ for all $n$, from continuity it follows that
 $\nabla S(\mathbf{x}_0)=\bu$ and hence $\mathbf{x}_0\in\mathcal{A}_{\bu}$.
 Let $\mathbf{p}_n \equiv \mathbf{x}_n-\mathbf{x}_0$ and $\mathbf{h}_n \equiv
 \frac{\mathbf{p}_n}{\|\mathbf{p}_n\|}$.  Then
\begin{equation}
\lim_{n\rightarrow\infty}\frac{\nabla S(\mathbf{x}_n) -\nabla
  S(\mathbf{x}_0)-\mathcal{H}_{\mathbf{x}_0}\mathbf{p}_n}{\|\mathbf{p}_n\|}=0
\end{equation}
where the linear operator $\mathcal{H}_{\mathbf{x}_0}$ is the \emph{Hessian}
of $S$ at $\mathbf{x}_0$ (obtained from the set of derivatives of the vector
field $\nabla S : \mathbb{R}^d \rightarrow \mathbb{R}^d$ at the location
$\mathbf{x}_0$). As $\nabla S(\mathbf{x}_n) = \nabla S(\mathbf{x}_0) = \bu$
and $\mathcal{H}_{\mathbf{x}_0}$ is linear, we get
\begin{equation}
\lim_{n\rightarrow\infty} \mathcal{H}_{\mathbf{x}_0} \mathbf{h}_n = 0.
\end{equation}
Since $\mathbf{h}_n$ is defined above to be a unit vector, it follows that
$\mathcal{H}_{\mathbf{x}_0}$ is rank deficient and
$\det\left(\mathcal{H}_{\mathbf{x}_0}\right) = 0$. Hence
$\mathbf{x}_0\in\mathcal{B}$ and $\bu\in\mathcal{C}$ resulting in a
contradiction.\\
\end{proof}

\noindent \emph{2. Proof of Neighborhood Lemma}
\begin{proof}
 Observe that the set $\mathcal{B}$ defined in (\ref{eq:setB}) is closed
 because if $\bxo$ is a limit point of $\mathcal{B}$, from the continuity of
 the determinant function we have $\det\left( \mathcal{H}_{\bxo} \right)=0$ and
 hence $\bxo\in\mathcal{B}$.  Being a bounded subset of $\Omega$, $\mathcal{B}$ is also
 compact. As $\partial \Omega$ is also compact and $\nabla S$ is continuous,
 $\mathcal{C}$ is compact and hence $\mathbb{R}^d-\mathcal{C}$ is open. Then
 for $\bu_0\notin\mathcal{C}$, there exists an open neighborhood
 $\mathcal{N}_{r}(\bu_0)$ for some $r>0$ around $\bu_0$ such that
 $\mathcal{N}_{r}(\bu_0)\cap\mathcal{C}=\emptyset$. By letting $\eta =
 \frac{r}{2}$, we get the required closed neighborhood
 $\overline{\mathcal{N}_{\eta}(\bu_0)} \subset \mathcal{N}_{r}(\bu_0)$
 containing $\bu_0$.

Since $\det\left(\mathcal{H}_{\bx}\right)\not=0, \forall \bx \in
\mathcal{A}_{\bu_{0}}$, 
points 1, 2 and 3 of this lemma follow directly from the inverse function
theorem.  As $|\mathcal{A}_{\bu_{0}}|$ is finite by Lemma~\ref{lemma:finitenessLemma}, the
closed neighborhood $\overline{\Neta(\bu_0)}$ can be chosen independently of
$\bx \in \mathcal{A}_{\bu_{0}}$ so that points 1 and 3 are satisfied $\forall \bx \in
\mathcal{A}_{\bu_{0}}$. In order to prove point 4, note that the eigenvalues of
$\mathcal{H}_{\bx}$ are all \emph{non-zero} and vary continuously for $\bx \in
\overline{\mathcal{N}_{\eta}(\bx)}$. As the eigenvalues never cross zero,
they retain their sign and so the signature of the Hessian stays fixed.\\
\end{proof}

\noindent \emph{3. Proof of Density Lemma}
\begin{proof}
 Since the random variable $X$ is assumed to have a uniform distribution on
 $\Omega$, its density at every location $\bx\in\Omega$ equals
 $\frac{1}{\mu(\Omega)}$. Recall that the random variable $Y$ is obtained via
 a random variable transformation from $X$, using the function $\nabla S$. The
 Jacobian of $\nabla S$ at a location $\bx \in \Omega$ equals the Hessian
 $\mathcal{H}_{\bx}$ of the function $S$ at $\bx$.  Barring the set
 $\mathcal{C}$ corresponding to the union of the image (under $\nabla S$) of
 the set of points $\mathcal{B}$ (where the Hessian vanishes) and the boundary
 $\partial \Omega$, the density of $Y$ exists on $\bu \in
 \mathbb{R}^d-\mathcal{C}$ and is given by (\ref{eq:graddensity}). Please
 see well known sources such as \cite{Billingsley95} for a detailed explanation. \\
\end{proof}

For the sake of completeness we explicitly prove the well-known result stated
in Integral Lemma~\ref{lemma:integralLemma}.\\
 \noindent \emph{4. Proof of Integral Lemma}
\begin{proof}
Define a function $H(\bx)$ by 
\[
H(\bx)\equiv\left\{ \begin{array}{ll}
1: & \mbox{if }\bx\in\Omega;\\
0: & \mbox{otherwise}.
\end{array}\right.
\]
 Let $f(\bx)=H(\bx)\exp\left(\frac{iS(\bx)}{\tau}\right)$. Then,
\begin{equation}
F_{\tau}(\bu)=\frac{1}{(2\pi\tau)^{\frac{d}{2}} \mu(\Omega)^{\frac{1}{2}}}\int
f(\bx)\exp\left(\frac{-i(\bu \cdot\bx)}{\tau}\right)d\bx.
\end{equation}
 Letting $\bv = \frac{\bu}{\tau}$ and $G(\bv)=F_{\tau}(\bu)$, we get
\begin{equation}
(\tau)^{\frac{d}{2}}\mu(\Omega)^{\frac{1}{2}}G(\bv)=\frac{1}{(2\pi)^{\frac{d}{2}}}\int
  f(\bx)\exp\left(-i \bv \cdot\bx \right)d\bx.
\end{equation}
 As $f$ is $\ell^{1}$ integrable, by Parseval's Theorem (see \cite{Bracewell99}) we have
\begin{equation}
\int\left|f(\bx)\right|^{2}d\bx=\tau^d \mu(\Omega) \int\left|F_{\tau}(\bv
\tau)\right|^{2}d\bv = \mu(\Omega) \int\left|F_{\tau}(\bu)\right|^{2}d\bu.
\end{equation}
 By noting that 
\begin{equation}
\int\left|f(\bx)\right|^{2}d\bx= \int\limits
_{\Omega}\left|\exp\left(\frac{iS(\bx)}{\tau}\right)\right|^{2}d\bx=
\mu(\Omega),
\end{equation}
 we get the desired result, namely
\begin{equation}
\int\left|F_{\tau}(\bu)\right|^{2}d\bu=1.
\end{equation}
 \end{proof}
 
 \noindent \emph{5. Proof of Cross Factor Nullifier Lemma}
 \begin{proof}
Let $p_{k,l}(\bu)$ denote the phase of the exponential in the cross term
(excluding the terms with constant signatures), i.e,
\begin{align}
\label{P} 
p_{k,l}(\bu) &= \Psi(\bx_k(\bu);\bu)-\Psi(\bx_l(\bu);\bu) \nonumber \\
&= S(\bx_k(\bu)) - S(\bx_l(\bu)) - \bu \cdot (\bx_k(\bu) - \bx_l (\bu)).
\end{align}
Its gradient with respect to $\bu$ equals
\begin{equation}
 \nabla p_{k,l} (\bu) = J_{\bx_k}[\nabla S(\bx_k(\bu))-\bu] - J_{\bx_l}[\nabla
   S(\bx_l(\bu))-\bu] - \bx_k(\bu) + \bx_l(\bu)
\end{equation} 
where $J_{\bx_k}$ is the Jacobian of $\bx(\bu)$ at $\bx_k$ whose $(i,j)^{\mathrm{th}}$
term equals $\frac{\partial x_j}{\partial u_i}$ (with a similar expression for $J_{\bx_k}$).  Since
$\nabla S(\bx_k(\bu)) = \nabla S(\bx_l(\bu)) = \bu$, we get $ \nabla p_{k,l}
(\bu) = \bx_l(\bu) - \bx_k(\bu) \not= 0$.  This means that the phase function
of the exponential in the statement of the lemma is \emph{non-stationary} and
hence \emph{does not contain any stationary points of the first kind}. Let
\begin{align}\label{Q}
q_{k,l}(\bu) = \frac{1}{\sqrt{|\det(\mathcal{H}_{\bx_k(\bu)})|}\sqrt{|\det
    (\mathcal{H}_{\bx_l(\bu)})|}}.
\end{align}
Since $\nabla p_{k,l} \not=0$, consider the vector field $\bf_{k,l}(\bu) =
\frac{\nabla p_{k,l}(\bu)}{\|\nabla p_{k,l}(\bu)\|^2}q_{k,l}(\bu) $ and as
before note that
\begin{align}
\label{eq:exprewrite}
\exp\left(\frac{i}{\tau}p_{k,l}(\bu)\right)q_{k,l}(\bu) &= i\tau \left[\nabla
  \cdot \bf_{k,l}(\bu)\right] \exp\left(\frac{i}{\tau}p_{k,l}(\bu) \right)
\nonumber \\ &- i\tau\nabla \cdot \left[\bf_{k,l}(\bu)
  \exp\left(\frac{i}{\tau} p_{k,l}(\bu)\right)\right]
\end{align}
where $\nabla\cdot$ is the divergence operator.  Inserting (\ref{eq:exprewrite})
in the second line of (\ref{PRIMARY}), integrating over
$\overline{\Neta(\bu_0)}$, and applying the divergence theorem we get
\begin{align}
\label{eq:integralafterdivergence}
\int\limits_{\overline{\Neta(\bu_0)}} \exp\left(\frac{i}{\tau}p_{k,l}(\bu)
\right) q_{k,l}(\bu) d\bu &= i \tau \int\limits_{\overline{\Neta(\bu_0)}}
\left[ \nabla \cdot \bf_{k,l}(\bu) \right] \exp\left(\frac{i}{\tau}
p_{k,l}(\bu) \right)d\bu \nonumber \\ &- i \tau \int\limits_{\partial
  \overline{\Neta(\bu_0)}} \left( \bf_{k,l} \cdot \mathbf{n}\right)
\exp\left(\frac{i}{\tau } p_{k,l}(\bu(\bv))\right)d\bv.
\end{align}
Here $\mathbf{n}$ is the unit outward normal to the positively oriented
boundary $\partial \overline{\Neta(\bu_0)}$ parameterized by $\bv$.  In the
right side of (\ref{eq:integralafterdivergence}), notice that all terms inside
the integral are bounded. The factor $\tau$ outside the integral ensures that
\begin{equation}
\lim\limits_{\tau \rightarrow 0} \int\limits_{\overline{\Neta(\bu_0)}}
\exp\left(\frac{i}{\tau}p_{k,l}(\bu) \right) q_{k,l}(\bu) d\bu =0.
\end{equation}
\end{proof}

\section{Well-behaved function on the boundary}
\label{sec:secondkindpoints}
One of the foremost requirements for Theorem~\ref{thm:GradDensity} to be valid
is that the function $\Psi(\bx;\bu)=S(\bx)-\bu \cdot \bx$ have a \emph{finite}
number of stationary points of the second kind on the boundary for almost all
$\bu$. The stationary points of the second kind are the critical points on the
boundary $\Gamma = \partial \Omega$ where a level curve $\Psi(\bx;\bu) = c$
touches $\Gamma$ for some constant $c$
\cite{Wong89,Wong81}. Contributions from the second kind are
generally $O\left(\tau^{\frac{d+1}{2}}\right)$, but an infinite number of them
could produce a combined effect of $O\left(\tau^{\frac{d}{2}}\right)$,
tantamount to a stationary point of the first kind \cite{Wong89}. If so, we
need to account for the contribution from the boundary which could in effect
invalidate our theorem and therefore our entire approach. However, the
condition for the infinite occurrence of stationary points of the second kind
is so restrictive that for all practical purposes they can be ignored. If the
given function $S$ is well-behaved on the boundary in the sense explained
below, these thorny issues can be sidestepped. Furthermore, as we will be
integrating over $\bu$ to remove the cross-phase factors, it suffices that the
aforementioned finiteness condition be satisfied for \emph{almost} all $\bu$
instead of for all $\bu$.

 Let the location $\bx \in \Gamma$ be parameterized by the variable $\by$, i.e.,
 $\bx(\by)$. Let $Q(\bx)$ denote the Jacobian matrix of $\bx(\by)$ whose
 $(i,j)^{\mathrm{th}}$ entry is given by
 \begin{align} 
\label{DEFQ}
Q_{ij}(\bx) = \frac{\partial x_j}{\partial y_i}
\end{align}
Stationary points of the second kind occur at locations $\bx$ where $\nabla
\Psi(\bx(\by);\bu) = 0$ which translates to
\begin{equation}
\label{eq:secondkindcondition}
Q(\bx)(\nabla S - \bu) =0. 
\end{equation}
This leads us to define the notion of a well-behaved function on the boundary.
\begin{definition}
\label{def:wellbehaved}
A function $S$ is said to be well-behaved on the boundary provided
(\ref{eq:secondkindcondition}) is satisfied only at a \emph{finite} number of
boundary locations for almost all $\bu \in \mathcal{C}$. \\
\end{definition}
Definition~\ref{def:wellbehaved} immediately raises the following
questions: (i) \emph{Why is the assumption of a  well behaved $S$ weak}? and (ii)
 \emph{Can the well-behaved condition imposed on $S$ be easily satisfied in
  all practical scenarios}?  Recall that the finiteness of premise
(\ref{eq:secondkindcondition}) entirely depends on the behavior of the
function $S$ on the boundary $\Gamma$. Scenarios can be manually handcrafted
where the finiteness assumption is violated and (\ref{eq:secondkindcondition})
is forced to satisfy at all locations. Hence it is meaningful to ask:
\emph{What stringent conditions are required to incur an infinite number of
  stationary points on the boundary?} We would like to convince the reader
that in all practical scenarios, $S$ will sustain only a finite number of
stationary points on the boundary and hence it is befitting to assume that the
function $S$ is well-behaved on the boundary.  The reader should bear in mind
that our explanation here is \emph{not} a formal proof but an intuitive
reasoning of why the well-behaved condition imposed on $S$ is reasonable.

To streamline our discussion, we consider the special case where the boundary
$\Gamma$ is composed of a sequence of hyper-planes as any smooth boundary can
be approximated to a given degree of accuracy by a finite number of
hyper-planes. On any given hyperplane, $Q(\bx)$ remains fixed. Recall that from
the outset, we omit the set $\mathcal{C}$ (i.e., $\bu \notin \mathcal{C}$)
which includes the image under $\nabla S$ of the boundary $\Gamma = \partial
\Omega$. Hence $\nabla S \not= \bu$ for any point $\bx \in \Gamma$ for $\bu
\notin \mathcal{C}$. Since the rank of $Q$ is $d-1$ and $\nabla S-\bu$ is
required to be orthogonal to all the $d-1$ rows of $Q$ for
condition~\ref{eq:secondkindcondition} to hold, $\nabla S-\bu$ is confined to
a 1-D subspace.  So if we enforce $\nabla S$ to vary smoothly on the
hyperplane and not be constant, we can circumvent the occurrence of an
infinite number of stationary points of the second kind for all
$\bu$. Additionally, we can safely disregard the characteristics of the
function $S$ at the intersection of these hyperplanes as they form a set of
measure zero.  To press this point home, we now formulate the worst possible
scenario where $\nabla S$ is a constant vector $\bt$. Let $\mathcal{D}$ denote
a portion of $\Gamma$ where $\nabla S = \bt$. Let $\bu = \bu_0$ and $\bu =
\bu_1$ result in infinite number of stationary points of the second kind on
$\mathcal{D}$. As $\nabla S-\bu$ is limited to a 1-D subspace, we must have
$\bt-\bu_1 = \lambda (\bt-\bu_0)$ for some $\lambda \not=0$, i.e, $\bu_1 =
(1-\lambda)\bt + \lambda \bu_0$. So in any given region of $\Gamma$, there is
at most a 1-D subspace (measure zero) of $\bu$ which results in an infinite
number of stationary points of the second kind in that region. Our
well-behaved condition is then equivalent to assuming that \emph{the number
  of planar regions on the boundary where $\nabla S$ is constant is finite}.

The boundary condition is best exemplified with a 2D example. Consider a line
segment on the boundary $x_2 = m x_1 + b$. Without loss of generality, assume
the parameterization $y = x_1$. Then $Q = \left[\begin{array}{c}1
    \\ m \end{array} \right]$. Equation~\ref{eq:secondkindcondition} can be
interpreted as $S_1 + m S_2 = u_1+mu_2$ where $S_i = \frac{\partial
  S}{\partial x_i}$. So if we plot the sum $S_1 + mS_2$ for points along the
line, the requirement reduces to \emph{the function $S_1 + mS_2$ not
  oscillating an infinite number of times around an infinite number of ordinate
  locations $u_1+mu_2$}. It is easy to see that the imposed condition is
indeed weak and is satisfied by almost all smooth functions allowing us to affirmatively conclude that
the enforced well-behaved constraint (\ref{def:wellbehaved}) does not impede the usefulness and application of our wave function method for estimating joint density of the gradient function.

\bibliographystyle{siam}
\bibliography{HigherDimensionGradDensityEstimation} 

\begin{thebibliography}{10}

\bibitem{Bertozzi07}
{\sc M.~Bertozzi, A.~Broggi, M.~Del Rose, M.~Felisa, A.~Rakotomamonjy, and
  F.~Suard}, {\em A pedestrian detector using histograms of oriented gradients
  and a support vector machine classifier}, in IEEE Conference on Intelligent
  Transportation Systems, 2007, pp.~143--148.

\bibitem{Billingsley95}
{\sc P.~Billingsley}, {\em Probability and measure}, Wiley-Interscience, New
  York, NY, 3rd~ed., 1995.

\bibitem{Bishop06}
{\sc C.M. Bishop}, {\em Pattern recognition and machine learning ({I}nformation
  science and statistics)}, Springer, New York, NY, 2006.

\bibitem{Bracewell99}
{\sc R.N. Bracewell}, {\em The {F}ourier transform and its applications},
  McGraw-Hill, New York, NY, 3rd~ed., 1999.

\bibitem{Cencov62}
{\sc N.~N. Cencov}, {\em Estimation of an unknown distribution density from
  observations}, Soviet Math., 3 (1962), pp.~1559--1562.

\bibitem{Dalal05}
{\sc N.~Dalal and B.~Triggs}, {\em Histograms of oriented gradients for human
  detection}, in IEEE Conference on Computer {V}ision and {P}attern
  {R}ecognition {(CVPR)}, 2005, pp.~886--893.

\bibitem{Fukunaga75}
{\sc K.~Fukunaga and L.~Hostetler}, {\em The estimation of the gradient of a
  density function with applications in pattern recognition}, IEEE Trans.
  Inform. Theory, 21 (1975), pp.~32--40.

\bibitem{Gurumoorthy12}
{\sc K.~S. Gurumoorthy and A.~Rangarajan}, {\em Distance transform gradient
  density estimation using the stationary phase approximation}, SIAM J. Math.
  Anal., 44 (2012), pp.~4250--4273.

\bibitem{Gurumoorthy14}
\leavevmode\vrule height 2pt depth -1.6pt width 23pt, {\em Errors bounds for
  gradient density estimation computed from a finite sample set using the
  method of stationary phase}, CoRR, abs/1404.1147 (2014).

\bibitem{Hu13}
{\sc R.~Hu and J.~Collomosse}, {\em A performance evaluation of gradient field
  {HOG} descriptor for sketch based image retrieval}, Comput. Vision Image
  Underst., 117 (2013), pp.~790--806.

\bibitem{Jones58}
{\sc D.S. Jones and M.~Kline}, {\em Asymptotic expansions of multiple integrals
  and the method of stationary phase}, J. Math. Phys., 37 (1958), pp.~1--28.

\bibitem{McClure91}
{\sc J.P. McClure and R.~Wong}, {\em Two-dimensional stationary phase
  approximation: Stationary point at a corner}, SIAM J. Math. Anal., 22 (1991),
  pp.~500--523.

\bibitem{McClure97}
\leavevmode\vrule height 2pt depth -1.6pt width 23pt, {\em Justification of the
  stationary phase approximation in time-domain asymptotics}, Proc. Math. Phys.
  Eng. Sci., 453 (1997), pp.~1019--1031.

\bibitem{Parzen62}
{\sc E.~Parzen}, {\em On the estimation of a probability density function and
  the mode}, Ann. Math. Stat., 33 (1962), pp.~1065--1076.

\bibitem{Rosenblatt56}
{\sc M.~Rosenblatt}, {\em Remarks on some nonparametric estimates of a density
  function}, Ann. Math. Stat., 33 (1956), pp.~832--837.

\bibitem{Rudin76}
{\sc W.~Rudin}, {\em Principles of mathematical analysis}, McGraw-Hill, New
  York, NY, 3rd~ed., 1976.

\bibitem{Scott79}
{\sc D.~Scott}, {\em On optimal and data-based histograms}, Biometrika, 66
  (1979), pp.~605--610.

\bibitem{Silverman86}
{\sc B.W. Silverman}, {\em Density estimation for statistics and data
  analysis}, Chapman and Hall/CRC, New York, NY, 1986.

\bibitem{Suard06}
{\sc F.~Suard, A.~Rakotomamonjy, and A.~Bensrhair}, {\em Pedestrian detection
  using infrared images and histograms of oriented gradients}, in IEEE
  Conference on Intelligent Vehicles, 2006, pp.~206--212.

\bibitem{Vapnik98}
{\sc V.N. Vapnik}, {\em Statistical Learning Theory}, Wiley-Interscience, New
  York, NY, 1998.

\bibitem{Wahba75}
{\sc G.~Wahba}, {\em Optimal convergence properties of variable knot, kernel,
  and orthogonal series methods for density estimation}, Ann. Stat., 3 (1975),
  pp.~15--29.

\bibitem{Wong89}
{\sc R.~Wong}, {\em Asymptotic approximations of integrals}, Academic Press,
  New York, NY, 1989.

\bibitem{Wong81}
{\sc R.~Wong and J.P. McClure}, {\em On a method of asymptotic evaluation of
  multiple integrals}, Math. Comp., 37 (1981), pp.~509--521.

\bibitem{Zhu06}
{\sc Q.~Zhu, M.-C. Yeh, K.-T. Cheng, and S.~Avidan}, {\em Fast human detection
  using a cascade of histograms of oriented gradients}, in IEEE Conference on
  Computer {V}ision and {P}attern {R}ecognition {(CVPR)}, 2006, pp.~1491--1498.

\end{thebibliography}

\end{document}